\documentclass[twocolumn]{svjour3}         % twocolumn
\smartqed  % flush right qed marks, e.g. at end of proof

\usepackage{amsmath,%amsthm,
}
\usepackage{dsfont}
\usepackage{graphicx}

\usepackage[all]{xy}

\newtheorem{thm}{Theorem}%[section]
\newtheorem{prop}{Proposition}
\newtheorem{lem}{Lemma}
\newtheorem{rem}{Remark}

\newcommand{\mathset}[1]{{\left\{#1\right\}}} 
\newcommand{\absolute}[1]{\left\lvert#1\right\rvert}
\newcommand{\norm}[1]{\left\|#1\right\|}

\DeclareMathOperator{\Trace}{Trace}
\DeclareMathOperator{\Rank}{Rank}

\begin{document}

\title{A dyadic solution of %stereo vision 
relative pose problems%\thanks{
%Grants or other notes
}


\author{Patrick Erik Bradley%First Author         \and
        %Second Author %etc.
}


\institute{P.E. Bradley \at
              Institut f\"ur Photogrammetrie und Fernerkundung, Universit\"at Karlsruhe (TH), Englerstr.\ 7, 76131 Karlsruhe, Germany \\
              \email{bradley@ipf.uka.de}           %  \\
           %\and
           %S. Author \at
           %   second address
}

\date{Received: date / Accepted: date}

\maketitle

\begin{abstract}
%Insert your abstract here. Include keywords, PACS and mathematical
%subject classification numbers as needed.
A hierarchical  interval subdivision is shown to lead to a $p$-adic encoding
of image data. This allows  in the case of the relative pose problem
in computer vision and photogrammetry 
to  derive equations having $2$-adic numbers as coeffients,
and to use the Hensel lifting method to their solution.
This method is applied to the linear and non-linear
equations coming from eight, seven or five point correspondences.
An inherent property of the method is its robustness.

\end{abstract}

%\section{Introduction}
%\label{intro}
%Your text comes here. Separate text sections with
%\section{Section title}
%\label{sec:1}
%Text with citations \cite{RefB} and \cite{RefJ}.
%\subsection{Subsection title}
%\label{sec:2}
%as required. Don't forget to give each section
%and subsection a unique label (see Sect.~\ref{sec:1}).
%\paragraph{Paragraph headings} Use paragraph headings as needed.
%\begin{equation}
%\end{equation}

%\begin{figure}
%  \includegraphics{example.eps}
%\end{figure}
%\begin{figure*}
%  \includegraphics[width=0.75\textwidth]{example.eps}
%\end{figure*}
%\begin{table}
%\begin{tabular}{lll}
%\hline\noalign{\smallskip}
%\end{tabular}
%\end{table}

\section{Introduction}
The issue of estimating camera motion from two views is classical by now, and 
methods from projective and algebraic geometry towards its solution
were employed at an early stage (e.g. \cite{Sturm1869}). The beginning of this present century
witnesses the application of sophisticated methods from computational
commutative algebra in order to rephrase the equations into a form from which
solutions can be  obtained with relative ease.
The relationship between the views is established by finding correspondences between 
point pairs taken from both images. The fundamental matrix
faithfully encodes the geometric 
relationship between
the two images.  For normalised cameras, the fundamental matrix coincides with 
the essential matrix.
In general, the two matrices are   related 
through  the camera calibration. Hence, if the calibration is known, it is sufficient
to estimate the essential matrix in order to solve the relative pose problem.
From a conceptional as well as a computational point of view, it makes sense to use
only few correspondences of image points in order to estimate  
the essential matrix $E$.  And different samples of $n$ correspondences
lead to a set of candidate essential matrices from which an optimal choice can
be made. 

Since $E=(e_{ij})$ is a projective $3\times 3$ matrix, i.e.\ only determined up to a scalar factor,
the number of parameters to be found is at most $8$. As each pair of corresponding
image points leads to a linear equation in the $e_{ij}$, it suffices to 
take $n=8$, as has been showed in \cite{8pt}. However, this linear method does not
take into account further constraints on $E$. So, a $7$-point algorithm was
developped \cite{7pt-a,7pt-b} (cf.\ also \cite{Zhang1998} for an overview). The
$5$-point algorithm of \cite{Nister5pt} uses 
the  minimal number of point
correspondences necessary
for determining $E$. The non-linear constraints lead to 
homogeneous equations of degree three in four unknowns which are transformed through an
elimination process into a univariate equation of degree $10$.
Hence the number of complex solutions is not more than $10$.
The elimination process can be simplified \cite{easy5pt},
or replaced by Gr\"obner basis methods \cite{SteweniusDiss}.
There is a lot of existing work towards
optimising the performance of this method in the computational algebraic geometry community. 
Stewenius et al.\ compare the performance  of elimination and Gr\"obner basis methods 
\cite{SEN2006}.

%From a theoretical point of view, solving a minimal problem gives a deeper understanding
%of the problem itself. Geometrically, it is about intersecting hypersurfaces in
%%projective $9$-space. Hence the justification of using complex algebraic methods.
%From an arithmetical point of view, it must be observed that the correspondences
%are usually taken from image points whose coordinates are rational numbers.
%This allows to employ methods with a more number-theoretic flavour.
%In \cite{NHSgalois}, number fields containing the solutions
%are discussed in order to prove the general non-solvability by radicals of occurring
%polynomials.

Correspondences often suffer a geometric perturbation, i.e.\
the correct point $u'$ in image $I'$ corresponding to point $u$ in
image $I$ is mostly found up to a small error vector $\epsilon$
such that $u'+\epsilon$ instead of $u'$ is matched to $u$. 
%This suggests to use a hierarchical approach to describing 
%point correspondences. 
In this article, %we will develop such 
a hierarchical method
based on interval subdivision is developped to 
handle this problem. A natural way of encoding hierarchies
is provided by the $p$-adic numbers, where $p$ is a fixed prime number. 
In our case, we can use $p=2$ and represent image points
by pairs of binary expansions
$$
a=\sum\limits_{n=0}^\infty a_i2^i
$$
with coefficients $a_i$ equal to $0$ or $1$. These expansions can be infinite,
theoretically. Practically, the finiteness of resolution means approximation
through truncation. The framework for this method is $p$-adic geometry
which has been applied in video segmentation and data analysis 
\cite{B-PXK2001,Murtagh2004,Brad-JoC%,Brad-pNUAA
}.

In the context of relative pose problems, we will use the so-called
{\em lifting method}
for solving the equations. This is provided by Hensel's lemma which says that
under certain conditions a solution   of a given equation 
modulo $p$ can
be expanded to a $p$-adic solution. 
Its proof is constructive. In fact, this is  the $p$-adic analogue
of a Newton iteration. Applied to the equations of the $n$-point problems
for $n=5,7,8$, we obtain the result that for many choices
of point correspondences, a $2$-adic solution can be constructed.
The $2$-adic essential matrix then allows to hierarchically 
approximate the rigid motion by truncation. A side effect
of the lifting method is its high robustness to geometric perturbations.
The encoding method  ensures further that the number of Newton iterations 
required is proportional to the order of resolution desired.

The article is structured as follows. The next two sections are a brief introduction to
$p$-adic numbers, and to Hensel's lemma
in a multivariate formulation. This is followed by a section 
on $2$-adic encoding of image pixels for a $p$-adic camera model.
The last section applies Hensel's lemma to  the $2$-adically defined equations 
for the eight-, seven-, and five-point problems.

\section{$p$-adic numbers}

%- Notation: $\mathds{Z}_p$ vs.\ $\mathcal{Z}_p$ 

%- avoid confusion: write $\mathds{F}_p$ and $\mathcal{Z}_{p^\ell}$

%\bigskip
The $p$-adic numbers were first described by Hensel in
\cite{pnum}. They are expansions of the form
\begin{align}
a=\sum a_ip^i \label{pexpansion}
\end{align}
into powers of a fixed prime number $p$ and coefficients $a_i\in\mathset{0,\dots,p-1}$.
If there are only finitely many terms in (\ref{pexpansion}),
this defines a rational number. Any natural number has a finite $p$-adic expansion
(\ref{pexpansion})
with $i\ge 0$. The important observation is that expansions with infinitely
many positive powers of $p$ have a meaning.
Namely, by defining
$$
\absolute{a}_p=p^{-\nu_a},
$$
where $\nu_a\in \mathds{Z}$ is the smallest exponent occurring in expansion (\ref{pexpansion}),
one obtains a metric for which the partial sums
$$
a_N=\sum\limits_{i\le N}a_ip^i
$$
converge to $a$:
$$
\absolute{a-a_N}_p=\frac{1}{p^{N+1}}\longrightarrow 0\quad\text{for}\quad N\longrightarrow \infty.
$$
The domain of all $p$-adic numbers is denoted by $\mathds{Q}_p$
and densly contains the rational numbers $\mathds{Q}$ with respect to this metric.
Those $p$-adic numbers $a$ with $\nu_a\ge 0$ are the $p$-adic integers,
denoted as $\mathds{Z}_p$. This domain contains densely the usual integers $\mathds{Z}$.
An equivalent description of $p$-adic integers is given by
$$
\mathds{Z}_p=\mathset{a\in\mathds{Q}_p\mid\absolute{a}_p\le 1}.
$$

Approximation of $p$-adic integers by their partial sums,  simply termed ``truncation'',
has an algebraic formulation. In its simplest form, a $p$-adic integer $a$
can be given by its coefficient $a_0\in\mathset{0,\dots,p-1}$,
and another  $p$-adic integer $b$ having the same coefficient
$b_0=a_0$ approximates $a$ up to that order of magnitude.
This is the case if and only if $a-b$ is divisible by $p$.
Hence, we arrive at
$$
\mathds{Z}_p/p\mathds{Z}_p\cong\mathds{Z}/p\mathds{Z}\cong\mathds{F}_p,
$$
where the latter is the finite field with $p$ elements.
Similarly, the algebraic formulation of approximation up to terms of higher order
is given by
$$
\mathds{Z}_p/p^{N}\mathds{Z}_p\cong\mathds{Z}/p^N\mathds{Z}=:\mathcal{Z}_{p^N}.
$$
In other words, congruences modulo $p^N$ yield finite approximations
of $p$-adic numbers by their partial sums up to terms of order $N$.
It is this property which makes $p$-adic numbers very suitable for hierarchically organised
data. In later sections, we will see how this algebraic formulation can be used in stereo vision.
A standard reference for $p$-adic numbers is \cite{Gouvea}.

\section{Hensel's lemma}

An important method in $p$-adic analysis is the so-called ``lifting''
of zeros of polynomials from $\mathds{F}_p$ to $\mathds{Z}_p$.
This uses the fact that 
$$
\mathds{Z}_p/p\mathds{Z}_p\cong\mathds{Z}/p\mathds{Z}\cong\mathds{F}_p
$$
by reducing a given equation $f(X)=0$ over $\mathds{Z}_p$ modulo $p$
to an equation $f\mod p$ over $\mathds{F}_p$.
 Hensel's lemma then gives a criterion when a solution $x_1$ of $f\mod p$
leads to a solution of $f$. This solution $\xi$, if it exists,
is constructed by an iterated process. Namely, first a solution $x_2$
of $f\mod p^2$ is constructed from $x_1$, and from this a solution $x_3$ of $f\mod p^3$ etc.
Each step of this iteration yields an approximation to the true solution
in $\mathds{Z}_p$ in the  sense
that 
$$
f(x_i)\equiv 0\mod p^{i},
$$
which in terms of the $p$-adic norm translates to
$$
\absolute{f(x_i)}<\frac{1}{p^i}.
$$
In other words, the sequence $f(x_i)$ converges $p$-adically to the
value $f(\xi)=0$. The construction process guarantees further that
$x_i$ also converges to the $p$-adic solution $\xi$.

\smallskip
We now state a multivariate Hensel's lemma,
but not in the most general form. For this, $\mathds{Z}[X_1,\dots,X_n]$
denotes the polynomial ring in $n$ variables with integer coefficients.

\begin{thm}[Multivariate Hensel's lemma] \label{hensel}
Let $$
f=(f_1,\dots,f_m)\in\mathds{Z}[X_1,\dots,X_n]^m
$$ 
with $m\le n$, and let $k\ge 2$.
Suppose that the vector $x=(x_1,\dots,x_n)\in\mathds{Z}^n$ is a solution of the congruences
\begin{align*}
f_1(X_1,\dots,X_n)&\equiv 0 \mod p^{k-1}\\
&\vdots\\
f_m(X_1,\dots,X_n)&\equiv 0 \mod p^{k-1}
\end{align*}
and that the matrix
$$
D_f(x)=\left(\frac{\partial f_i}{\partial X_j}\right)
$$
is  modulo $p$ of rank $m$.
Then there exist $t=(t_1,\dots,t_n)\in \mathset{0,\dots,p-1}^n$ such that
$$
f(x+p^{k-1}t)\equiv 0 \mod p^k.
$$
In particular, the equations $f(X)=0$ have a solution in $\mathds{Z}_p$.
\end{thm}

\begin{proof}
Consider the linear part in the Taylor expansion of $f$ in $x$:
$$
f(X)=f(x)+D_f(x)\cdot(X-x)+\text{terms of higher order},
$$ 
where $X=(X_1,\dots,X_n)$ denotes the vector of variables.
By assumption, it holds true that
$$
f(x)=p^{k-1}\cdot a
$$ 
for some vector $a\in \mathds{Z}^m$.
Due to the rank condition, the system of congruences
\begin{align}
a+D_f(x)\cdot t\equiv 0 \mod p \label{crucialcong}
\end{align}
has a solution vector $t\in\mathset{0,\dots,p-1}^n$.
Hence,
\begin{align*}
f(x+p^{k-1}t)&\equiv f(x)+D_f(x)\cdot p^{k-1}t
\\
&\equiv p^{k-1}(a+D_f(x)\cdot t)\equiv 0 \mod p^{k}.
\end{align*}
Iteration proves the last assertion.
\end{proof}

Usually, Hensel's lemma is stated in the case of a single univariate polynomial $f$
having modulo $p^{k-1}$ a zero $x$.
The rank condition translates to $f'(x)\not\equiv 0\mod p$, i.e.\ $x$ is a simple zero of $f$
modulo $p$. %For lack of reference, we give a proof in this general case.
The proof is a $p$-adic analogue of a Newton iteration.
From (\ref{crucialcong}) it follows for $m=n$ that the rank condition
implies uniqueness of the lift.
The univariate case can be found e.g.\ in \cite{Gouvea}.

\section{$p$-adic projective cameras}

\subsection{$p$-adic encoding of image pixels}

%- $2$-adic encoding from quadtree 

%- $2$-adic nearness in image implies Euclidean nearness

%- Does $2$-adic nearness  imply Euclidean nearness for essential matrices?

%\bigskip
Assume that a rectangular 2D image $I$ is given by $M\times N$ pixels and that
$m,n$ are minimal such that
$M\le 2^m$, $N\le 2^n$.
We will use a binary encoding of a given image point with pixel coordinates
$$
(x,y)\in \mathset{0,\dots,M-1}\times\mathset{0,\dots,N-1}.
$$ 
This means that $x$ and $y$ will be represented by a pair of binary
numbers obtained in a hierarchical manner.
Namley, consider one coordinate at a time, say $x$.
It can be arrived at by a sequence of iterated  subdivisions
of the interval ${0<\dots<2^{m}-1}$ into intervals of equal length.
After $m$ iterations each interval obtained contains precisely one pixel $x$-coordinate.
This means that a given value of $x$ lies in a uniquely determined nested sequence of 
intervals produced by this subdivision process.
The intervals form a rooted binary tree whose leaves
correspond uniquely to the $x$-coordinates of pixels.
By assigning for a given interval its left half the value $0$, and $1$ for its
right half, we obtain the binary representation
$$
r(x)= \sum\limits_{\nu=0}^{m-1}\alpha_\nu2^\nu,\quad\alpha_\nu\in\mathset{0,1}
$$
by traversing the path from root down to $x$ and picking up the zeros and ones along the way.

We will interpret the natural number $r(x)$ as a $2$-adic integer: $r(x)\in\mathds{Z}_2$. 
Each partial sum
$$
r_\ell=\sum\limits_{\nu=0}^\ell\alpha_\nu2^\nu
$$ 
is a $2$-adic approximation
of $r(x)$. Its error is bounded by the $2$-adic distance
$$
\absolute{r_\ell-r(x)}_2.
$$
In fact, $r(x)$ itself is a $2$-adic approximation to some imaginary ``pixel'' $\xi$
of infinite precision. The number $m$ is given by the degree of resolution
with which $\xi$ is viewed on the given image.
And an infinitely precise pixel would have a $2$-adic expansion with possibly infinitely
many coefficients equal to $1$.

\smallskip
We proceed similarly for the $y$-coordinate, and obtain an encoding
$$
c_2\colon I\to\mathds{Z}_2\times\mathds{Z}_2,\quad (x,y)\mapsto (r(x),s(y)),
$$
where $s$ is the binary encoding of the $y$-coordinate.

\medskip
Any $p$-adic encoding of a number $x$ bears the problem that its $p$-adic
approximations $r_\ell$ are determined by arithmetic properties of $x$,
ignoring their (euclidean) geometric properties. This means that
although $x_\ell$ could be   $p$-adically very close to $x$, it could be
at large distance for the euclidean norm.
Hence, arbitrary $p$-adic encodings of pixel data would be quite unsuitable
for many applications in computer vision.
However, the binary encoding through interval splitting does not
suffer this disadvantage.
Namely, the map $c_2$ respects the image geometry in the following sense:

\begin{prop}
The map 
$$
\iota_2\colon \sum\limits_{\nu=0}^{\infty}\alpha_\nu2^\nu\mapsto\sum\limits_{\nu=0}^{\infty}\alpha_\nu2^{-\nu}
$$
yields an inclusion of $\mathds{Z}_2$
into the real interval $[0,2)$. It has the 
property
$$
\absolute{r-s}_2<2^{-\ell}\Rightarrow
\absolute{\iota_2(r)-\iota_2(s)}_\infty< 2^{-\ell}
$$
for all $\ell\in\mathds{N}$.
\end{prop}

\begin{proof}
Any sum of negative powers of $2$ can be majorised by 
a (possibly shifted) geometric sum. Hence, such a sum converges for the real metric. 
This implies that the map $\iota_2$ is well-defined for arbitrary $2$-adic numbers.

The inclusion property is clear, at least for finite $2$-adic expansions. 
However,  the map $\iota_2$ is injective on all of $\mathds{Z}_2$.
Otherwise, assume that 
$$
\iota_2(a)=\iota_2(b)
$$
for some $a\neq b$. Then $a=\sum a_i2^i$ and $b=\sum b_i2^i$ must differ in some coefficient:
$$
a_n\neq b_n.
$$
Let $n$ be minimal such that this occurs. Then
it follows that
$$
\absolute{\iota_2(a)-\iota_2(b)}_\infty\ge\frac{1}{2^N},
$$
because the inequality holds true for all partial sums with at least $N$ terms.
This is a contradiction. Hence, $\iota_2$ is injective.

ssume $\absolute{r-s}_2<2^\ell$.
This means that
the $2$-adic expansions of  $r$ and $s$ have the first $\ell+1$ terms in common. 
The last assertion follows from this.
\end{proof}

As a consequence, we obtain an embedding of the image $I$ into $[0,2)^2$
via composition of $c_2$ and $\iota_2\times \iota_2$.

\subsection{$p$-adic camera model}

Recall that a projective camera is a map between projective
spaces
$$
\kappa\colon\mathds{P}^3\to\mathds{P}^2
$$
given by a $3\times 4$-matrix of rank $3$.
Usually, cameras are modelled as being defined over the real numbers,
i.e.\ they are given by real $3\times 4$-matrices.
However, when dealing with  data coming from  3D-objects from the real world,
cameras are usually approximated by rational matrices.
In this way, we arrive at rational cameras:
$$
\kappa\colon\mathds{P}^3(\mathds{Q})\to\mathds{P}^2(\mathds{Q})
$$
as maps between the rational points of projective spaces.
For such cameras, the real completion is only one choice of many.
Hence, $p$-adic cameras 
$$
\kappa\colon\mathds{P}^3(\mathds{Q}_p)\to\mathds{P}^2(\mathds{Q}_p)
$$
could also be considered in order to use methods from $p$-adic geometry
for camera computations. 

In fact, since finite resolution prevents stereoscopic vision at arbitrarily large 
distance, 
it can become possible to shift coordinate systems in such a way that
cameras are described by matrices whose entries are natural numbers,
at least approximately.
In the following subsection, we will show how this said approximation can
be done for the $2$-adic norm. Precisely the hierarchical method
used in the previous subsection in 2D extends in a natural way to 3D
in order to arrive at
$2$-adic camera model
$$
\kappa\colon\mathds{P}^3(\mathds{Z}_2)\to\mathds{P}^2(\mathds{Z}_2).
$$

\subsection{How to interpret a $2$-adic essential matrix}

Assume that through correspondences between some points $u\in I$, $u' \in I'$ 
in $2$-adic camera images $I,I\cong\mathds{P}^2(\mathds{Z}_2)$,
 a $2$-adic essential matrix 
$E\in\mathds{Z}_2^{3\times 3}$ has been produced. This matrix is the limit of 
matrices $E_\nu\in\mathcal{Z}_{2^\nu}^{3\times 3}$. 
Through its factorisation 
$$
E_\nu=T_\nu R_\nu
$$
into a skew-symmetric matrix $T_\nu$ and a rotation $R_\nu$
it allows to determine a point $U_\nu\in\mathds{P}^3(\mathcal{Z}_{2^\nu})$
of which the two cameras obtain the image points $u_\nu,u'_\nu \mod 2^\nu$.
These are related through $E_\nu$:
$$
u_\nu^T\cdot E_\nu\cdot u'_\nu=0.
$$ 
The $U_\nu$ converge $2$-adically to a point $U\in\mathds{P}^3(\mathds{Z}_2)$,
and the cameras
$$
\kappa,\kappa'\colon\mathds{P}^3(\mathds{Z}_2)\to\mathds{P}^2(\mathds{Z}_2)
$$
yield $\kappa(U)=u,\kappa'(U)=u'$.
In the same way as $u_\nu$ is an approximation of pixel $u$ at resolution $\nu$,
the 3D-point $U_\nu$ is an approximation of voxel $U$ at resolution $\nu$,
because we can apply the map $\iota_2$ in precisely the same manner
also to three dimensions to obtain the inclusion
$
\mathds{P}^3(\mathds{Z}_2)\to [0,2)^3.
$

%\section{Equations from $2$-adic image data in two views}

\section{Reconstruction from  point correspondences}

We will in the following content ourselves with solving the reconstruction problem
with two calibrated cameras by finding (possible candidates for) essential matrices
in the $2$-adic setting.

%\subsection{Lifting solutions of  linear equations}
%- rank issues

%- Hensel implies liftability

%- brute force lifting vs.\ linear solving for lifts

%- structure of $E$

\bigskip
Given $n$ correspondences,
we arrive at the system of linear equations
\begin{align}
u_i^TE\,u_i',\qquad i=1,\dots,n. \label{essentialcond}
\end{align}
For $n\le 9$, we can use the Hensel lifting method to
solve these linear equations, written as
\begin{align}
Ax=0 \label{linearequation}
\end{align}
with coefficient matrix $A\in\mathds{Z}_2^{n\times 9}$. Due to finite resolution
and our coding method,
the matrix entries lie in some set $\mathset{0,\dots,2^\nu-1}$.
Hence, the matrix does not change its shape modulo $2^\nu$. 

Assume that we are given a basis 
$B=(b_1,\dots, b_m)$ in normal form of the solution space of (\ref{linearequation})
modulo $2$. We will interpret $B$ as an ordered set, but 
also as a matrix whose columns are the vectors $b_1,\dots, b_m$.

\begin{thm}[Linear Hensel's lemma] \label{nptlift}
If the rank $m$ of $A\mod 2$ equals $\Rank(A)$,
%for all rows $a_i$ of the matrix $A$ it holds true that
%$$
%a_i\not\equiv 0\mod 2,
%$$
then   $B$ lifts to a set of linearly independent solutions of (\ref{linearequation})
with natural numbers as entries. 
\end{thm}

\begin{proof}
The linear equation (\ref{linearequation}) is equivalent to a linear equation 
%$$
%Tx=0
%$$
in staircase normal form. This normal form can be viewed as a system of
$\Rank(A)$
equations in the $9$ unknowns. By assumption, $m=\Rank(A)$. Hence, Theorem \ref{hensel} applies,
and $B$ lifts to a set of $m$ solution vectors 
$$
\tilde{B}=\mathset{\tilde{b}_1,\dots,\tilde{b}_m}
$$
to (\ref{linearequation}).

\medskip
Let us re-examine the proof of Theorem \ref{hensel} in order to see that  $\tilde{B}$
is linearly independent.
First, notice that $B$ is read off the staircase normal form  of $A \mod 2$.
This means that each $b\in B$ has  in some row $j_b$ entry $1$, and below it all entries are zero.
Further, all other entries of row $j_b$ of matrix $B$ are zero. The sequence $(j_b)_{b\in B}$ is strictly increasing
with respect to the order of occurrence of $b$ in $B$.
Now, a lift of $b$ to $b^{(k)}$ modulo $2^k$ has the property that the entry in row $j_b$ is 
odd, and all entries below are even. Likewise, all other entries of 
row $j_b$ in a lift of $B$ to $B^{(k)}$ modulo $2^k$ are even.
This description shows that the rows given by the sequence $(j_b)$ form a submatrix 
of $B^{(k)}$ having rank $m$.
Since this holds true for all $k>0$, it follows that $\tilde{B}$ is linearly independent.

The lifting process stops after a finite number of steps, because of our initial assumptions.
\end{proof}

%The requirement now is that the rank of the coefficient matrix   be $n$.
%In fact, by Theorem \ref{hensel}, it suffices to test the rank requirement
%modulo $2$, which means a significant computational improvement 
%in comparison with the analogue problem over the rational numbers.

\subsection{Reconstruction from $8$ points}

%\begin{rem}
Assume now that the rank of $A\mod 2$ be $n$, and that a basis $B=\mathset{b_1,\dots,b_{9-n}}$
of the solution space be given in normal form. Then, by Theorem \ref{nptlift}, 
finding a lift to $\mathcal{Z}_{2^\nu}$ of $B$
yields a basis of the solution space of (\ref{linearequation}). 
%In each step of the $2$-adic Newton iteration,
%a congruence $A\cdot t\equiv a\mod 2$ needs to be solved.
%For $p=2$  the brute force method would require at most $\nu\cdot 2^9=\nu\cdot 512$
%trials for the lift of each vector $b_i$.

As an application, we obtain a  $2$-adic version of the $8$-point algorithm of 
\cite{8pt}, simply by setting $n=8$. However, 
in the same way as its original, this ignores the rank constraint
$\det(E)=0$ for the essential matrix. Hence, we obtain the result:
%\end{rem}

\begin{thm}
Under the assumptions above for  
the matrix $A\in\mathds{Z}_2^{8\times 9}$, % satisfies the condition of Theorem \ref{nptlift},
the corresponding $8$-point problem has a unique solution $E$,
if additionally  $\Rank(E)=2$. If it can be assumed that $A\in\mathcal{Z}_{2^N}$,
then  $E$ is computed after $N-1$ iterations from $E\mod 2$. 
\end{thm}

\begin{proof}
%The asser
Since $\Rank(A\mod 2)=8$, there is one solution basis vector $b\in\mathcal{Z}^9$.
By Theorem \ref{nptlift}, it lifts to a non-trivial solution $e$ of $Ax=0$.
If the the matrix $E$ corresponding to $e$ is of rank $2$, it is the unique solution
to the $8$-point problem. The last assertion is an immediate consequence of
the %constructive nature 
 proof of Theorem \ref{nptlift}.
% number of iterations is determined by on the $2$-adic expansion of
%the entries of $A$.
\end{proof}

\subsection{Reconstruction from $7$ points}

The $7$-point method by \cite{7pt-a,7pt-b} yields $7$ linear constraints 
(\ref{essentialcond}) plus the cubic constraint
\begin{align}
\det(E)=0.\label{detcond}
\end{align}
Let us write that system of equations as $f(X)=0$.
If the $7$ points are sufficiently in general position, then
the rank of 
$D_f(x)\mod 2$ is $8$ for some solution $x$ modulo $2$:
\begin{align}
f(x)\equiv 0\mod 2, \label{opencond1}
\end{align}
and we can lift to a $2$-adic solution.
The reason is: 

\begin{lem} 
For $h:=\det(E)$ it holds true that
\begin{align}
D_h(e)\not\equiv 0\mod 2,\label{opencond2}
\end{align}
if $e\in \mathds{Z}_2^{3\times 3}$ is a sufficiently general instance of $E$.
\end{lem}

\begin{proof}
The polynomial $h$ is of degree three in the variables given by the entries of $E$.
Let $x$ be such an entry.
Then $h_x:=\frac{\partial h}{\partial x}$ is a polynomial of degree $2$.
Modulo $2$ only those terms of $h_x$ vanish which are of the form $2xy$ for some other variable $y\neq x$. 
Since $h=\det(E)$ is a square-free polynomial, this can never happen.
\end{proof}

In $2$-adic analytic geometry, the conditions (\ref{opencond1}) and (\ref{opencond2})
define an open subset in the space of all $3\times 3$-matrices with entries in $\mathds{Z}_2$.
This can be seen easily by translating (\ref{opencond1}) to the
inequality
$$
\absolute{f(x)}_2<1,
$$
and dealing similarly with (\ref{opencond2}).

We can go further and derive explicit conditions for the existence of $\mathds{Z}_2$-rational
solutions. Namely, write down the $1$-parameter solution of the linear equations as
$$
E=xE_1+(1-x)E_2,
$$
and obtain a polynomial 
$$
\det(E)=h(x)=ax^3+bx^2+cx+d
$$ 
 of degree $3$. %The first condition is that $h(x)\mod 2$ is of degree $3$, i.e.
%\begin{align}
%a\equiv 1\mod 2. \label{deg3cond}
%\end{align}
Now consider $h(x)\mod 2$. In case $d\equiv 0\mod 2$, zero is a
simple zero in $\mathds{F}_2$ if and only if
\begin{align}
c\equiv 1\mod 2.
\label{simpzerocond1}
\end{align}
In case $d\equiv 1\mod 2$, one is a simple zero in $\mathds{F}_2$ if and only if
\begin{align}
b\equiv c\equiv 0\mod 2\label{simpzerocond2},
\end{align}
in case $a\equiv 1\mod 2$,
and
\begin{align}
b\equiv0,\quad c\equiv 1\mod 2 \label{simpzerocond3}
\end{align}
otherwise.

\begin{thm} \label{2-adic-7pt}
The $7$-point problem has a $\mathds{Z}_2$-rational solution for many choices of $7$ point correspondences. Concretely, this is the case  if (\ref{simpzerocond1}),
(\ref{simpzerocond2}), or (\ref{simpzerocond3}) hold true
in their respective cases, together with the requirement that the rank of $E$ be
precisely $2$.
\end{thm}

\begin{proof}
Solve the $7$-point problem modulo $2$, and lift whenever  possible
as discussed above.
\end{proof}

\begin{rem}
Due to non-linearity of the constraints, we cannot expect anymore to
be able to lift solutions modulo $2$ to solutions which are defined over the natural numbers.
In other words, generally, the iteration will never stop unless an order of resolution
is specified. Then, a $2$-adically approximate solution will be obtained.
This is not different from the classical situation over the real numbers.
\end{rem}

%\subsection{Statistics with $n\ge 8$}

%How about some $2$-adic statistics?

\subsection{Solving the $5$-point non-linear equations}

Given five corresponding pairs of image points yields 
(\ref{essentialcond}) with $n=5$. If the rank of the corresponding matrix is five, 
the general solution can be written as
\begin{align}
E=xE_1+yE_2+zE_3+wE_4 \label{lincomb}.
\end{align}
Inserting this into the trace condition \cite{tracecondref}
\begin{align}
2\cdot EE^TE-\Trace(EE^T)\cdot E=0
 \label{tracecond}
\end{align}
yields $9$ cubic equations in four variables. We follow the
easily understandable 
method of
elimination via hidden variables used by \cite{easy5pt},
and obtain the linear  equation
$$
C(z)\cdot X=0,
$$
after setting $w=1$. 
Here,  
$$
X=(x^3, y^3,x^2y,xy^2,x^2,y^2,xy,x,1)
$$
is the vector of monomials,
and the entries of $C(z)$ are polynomials in $z$.
Now,  one seeks $z$ such that 
\begin{align}
\det(C(z))=0. \label{deg10poly}
\end{align}
The left hand side turns out to be a polynomial of degree $10$.

\smallskip
We will use the same lax formulation of our theorem as for the $7$-point problem,
but will be more specific in the proof:

\begin{thm}
There exists  a $\mathds{Z}_2$-rational solution to the $5$-point problem
for many choices of $5$ corresponding pairs of points.
\end{thm}

\begin{proof}
Let 
$$
g(z):=\det(C(z))=\sum\limits_{i=0}^{10}a_iz^i.
$$
%Then $\det(C(z))\mod 2$ is  of degree 10 if and only if $a_{10}\equiv 1\mod 2$.
%Given that, w
We consider two cases.

First, assume $a_0\equiv 0\mod 2$. In this case, 
$g(0)\equiv 0\mod 2$, and $z=0$ is a simple zero modulo 2 if and only if 
$a_1\equiv 1\mod 2$.

Secondly, if $a_1\equiv 1\mod 2$, then $f(1)\equiv 0\mod 2$ if and only if 
the number of odd coefficients in $g(z)$ is even.
Since
$$
f'(z)\equiv a_9z^8+a_7z^6+a_5z^4+a_3z^2+a_1\mod 2,
$$
$1$ is a simple zero modulo 2 if and only if  in addition
the number of odd coefficients in $f'(z)$ is odd.
\end{proof}

\begin{rem}
Hensel's lemma can be interpreted in this context as a $p$-adic stability
result. Namely, assume that, due to correspondence error, a given choice
of $n$ points yields perturbed equations 
$$
f(x)+\epsilon(x)=0,
$$
where $f(x)$ contain the ``true'' coefficients perturbed by some noise
coefficients contained in $\epsilon(x)$
with $2$-adic maximum norm
$$
\norm{\epsilon}_2=\max\mathset{\absolute{\epsilon_i}_2\mid i=1,\dots,m}\le 2^{-N}, 
$$
where $m$ is the number of noise coefficients $\epsilon_i$.
This means that the noise coefficients satisfy the congruence
$$
\epsilon_i\equiv 0\mod 2^N.
$$
Hence, the first $N$ iterations of Hensel lifting will lead to identical 
approximations to the solution of the unperturbed equations
$$
f(x)=0.
$$
As the $2$-adic encoding comes from interval subdivisions,
this observation implies a greater stability 
in comparison with the classical   approach over the real numbers.
In particular, the existence of a liftable solution modulo $2$
is not affected by perturbations $\epsilon$
with $\norm{\epsilon}_2\le\frac{1}{2}$.
This is definitively in contrast to the situation over the real numbers,
where small perturbations in coefficients can drastically change the number of  
real solutions. That issue is addressed for the 5-point relative pose problem e.g.\ 
in \cite{alt5pt}. 
\end{rem}

\section{Conclusion}

%- advantage of hierarchical encoding: complexity $=\log_2(\text{number of pixels})$

%- chance of $\mathds{Z}_2$-rational solution should be same as that for real solution

%- in presence of errors also?

%- performance?

%- future work: practical eval, Groebner, ransacp

%\bigskip
An encoding scheme for image pixels 
 through  hierarchical
 interval subdivision is proposed. This allows a $2$-adic encoding of pixels driven by geometry.
As an application to stereo vision, the 8-, 7- and 5-point equations are formulated
with coefficients from the ring $\mathds{Z}_2$ of $2$-adic integers.
These polynomial equations are solved using some multivariate forms of Hensel's
lemma. The essential matrices are obtained in the form of sequences of 
matrices modulo powers of $2$, corresponding to $2$-adic approximations to
the exact solutions in $\mathds{Z}_2$, whenever these exist.
Conditions on coefficients of equations modulo $2$ implying the existence
of $\mathds{Z}_2$-rational solutions are derived.
One feature of the hierarchical encoding is that the number
of iterations in solving the linear parts of the equations
is logarithmic in the number of pixels. 
Also the precision in the solution
of the non-linear equations is directly related to the desired resolution in 3D.
An immediate consequence of Hensel's lemma is that $p$-adically small perturbations
of the equations do  not affect the first approximations to their solution.
Further, the existence of liftable solutions is not affected by relatively large
perturbations.
This indicates a greater computational benefit  from
 the $2$-adic approach %is computationally more advantageous
compared to the classical approach using computational complex  algebraic
geometry before discarding non-real solutions.

\begin{acknowledgements}
%The author acknowledges support from 
% the Deutsche Forschungsgemeinschaft in project BR 3513/3-1. 
Sven Wursthorn is   thanked for the introduction into this fascinating topic,
and Boris Jutzi for multiple fruitful discussions.
\end{acknowledgements}



\end{document}